\documentclass[11pt]{article}

\usepackage{amsmath}
\usepackage{amsthm}
\usepackage{amssymb}
\usepackage{algorithm}
\usepackage{color}
\usepackage{graphicx}
\usepackage{wrapfig,epsfig}
\usepackage{url}
\usepackage{algpseudocode}
\usepackage[T1]{fontenc}
\usepackage{comment}
\usepackage{thm-restate}
\usepackage{bm}
\usepackage{tcolorbox}
\usepackage{subcaption}
\usepackage{enumitem}
\usepackage[margin=1in]{geometry}
\usepackage{bbm}

\newtheorem{theorem}{Theorem}[section]
 
\newtheorem{lemma}[theorem]{Lemma}

\theoremstyle{definition}

\newtheorem{definition}[theorem]{Definition}

\DeclareMathOperator{\E}{\mathbb{E}}

\newcommand{\eps}{\epsilon}
\newcommand{\wt}{\widetilde}

\newcommand{\KL}{\mathsf{KL}}

\newcommand{\mK}{\mathcal{K}}
\newcommand{\poly}{\mathsf{poly}}
\newcommand{\Ent}{\mathsf{Ent}}
\newcommand{\unif}{\mathsf{unif}}
\newcommand{\DCE}{\mathsf{DCE}}

\newcommand{\smooth}{\mathsf{smooth}}
\newcommand{\sml}{\mathsf{small}}

\definecolor{DarkBrown}{rgb}{0.4,0.13,0.13}

\title{High dimensional online calibration in polynomial time}
\author{Binghui Peng\\
Stanford University\\
\texttt{binghuip@stanford.edu}}
\date{\today}

\begin{document}
\maketitle

\begin{abstract}
In online (sequential) calibration, a forecaster predicts probability distributions over a finite outcome space $[d]$ over a sequence of $T$ days, with the goal of being calibrated. While asymptotically calibrated strategies are known to exist, they suffer from the curse of dimensionality: the best known algorithms require $\exp(d)$ days to achieve non-trivial calibration.

In this work, we present the first asymptotically calibrated strategy that guarantees non-trivial calibration after a polynomial number of rounds. Specifically, for any desired accuracy $\epsilon > 0$, our forecaster becomes $\epsilon$-calibrated after $T = d^{O(1/\epsilon^2)}$ days. We complement this result with a lower bound, proving that at least $T = d^{\Omega(\log(1/\epsilon))}$ rounds are necessary to achieve $\eps$-calibration. Our results resolve the open questions posed by \cite{abernethy2011does,hazan2012weak}.

Our algorithm is inspired by recent breakthroughs in swap regret minimization \cite{peng2023fast,dagan2024external}. Despite its strong theoretical guarantees, the approach is remarkably simple and intuitive: it randomly selects among a set of sub-forecasters, each of which predicts the empirical outcome frequency over recent time windows.
\end{abstract}
\clearpage
\newpage

\section{Introduction}
\label{sec:intro}
In online forecasting, a forecaster aims to predict the probability distribution of outcome over a sequence of $T$ days. On each day $t \in [T]$, the forecaster outputs a distribution $p_t\in\Delta_{d}$ over the outcome space $[d]$, and then observes the reazlied outcome $X_t \in [d]$.

A widely used metric for evaluating a forecasting performance is {\em calibration}. Informally, calibration assesses how well the predicted distributions align with actual outcomes over time: whenever the forecaster predicts a distribution $p$, the empirical distribution of outcomes on such days should be close to $p$. Formally, the total calibration error is defined as:
\begin{align*}
\text{calibration-error} := \sum_{p \in \Delta_d} \sum_{t \in [T]} \Big\| (p - X_t) \cdot \mathsf{1}[p_t = p] \Big\|_1.
\end{align*}
A forecaster is said to be $\eps$-calibrated if its total calibration error is at most $\eps T$. 
Intuitively, this means that the predictions are, on average, $\eps$-close to the observed outcomes.

Calibrated forecasting has a rich history, tracing back to the foundational work of \cite{brier1950verification,dawid1982well,oakes1985self}. 
The first algorithm for calibrated forecasting was introduced in the seminal work of \cite{foster1997calibrated,foster1998asymptotic}, which provided an asymptotically calibrated forecasting strategy: The forecaster becomes $\eps$-calibrated after $T = (\frac{1}{\eps})^{\Theta(d)}$ days. 
Remarkably, this guarantee holds even when the outcomes are adversarially chosen, and without any prior knowledge of the outcome distribution.



Beyond asymptotic guarantees, one seeks to design forecasting strategies that are as statistically efficient as possible. 
For the important case of binary outcome ($d = 2$), a recent line of work \cite{qiao2021stronger,dagan2024breaking} have focused on obtaining the optimal statistical complexity and \cite{dagan2024breaking} prove that the optimal calibration error is within $[T^{0.543}, T^{2/3-\delta}]$, where $\delta > 0$ is an absolute constant.

\paragraph{High dimensional calibration} 
This focus of this work is on high dimensional (or multi-class) calibration, where the outcome space consists of $d > 2$ possible values. Such settings frequently arise in applications like image classification~\cite{guo2017calibration}, next-token prediction~\cite{jiang2021can}, and strategic forecasting in games~\cite{foster1997calibrated}.
In this regime, all known algorithms~\cite{foster1997calibrated, abernethy2011does, hart2022calibrated, fishelson2025full} suffer from an {\em exponential dependence on $d$}: achieving $\eps$-calibration requires at least $T = (\frac{1}{\eps})^{\Omega(d)}$ days.

Does an efficient high dimensional calibrated forecasting strategy exist? This was posed as an open problem at COLT 2011 by \cite{abernethy2011does}. Subsequently, \cite{hazan2012weak} proved that no polynomial time {\em deterministic} forecasting strategy can guarantee {\em $\eps$-weak calibration} for $\eps = 1/\poly(d)$, assuming $\mathsf{PPAD} \not\subseteq \mathsf{RP}$. 
This result crucially depends on determinism: while weak calibration admits deterministic strategies \cite{kakade2008deterministic}, deterministic strategies for sequential calibration simply do not exist (regardless of runtime) \cite{oakes1985self}.
The statistical complexity of {\em randomized} sequential calibration is left as an open question by \cite{hazan2012weak}, and since then, no significant progress has been made.

The core challenge of high dimensional calibration lies in the exponential size of the prediction space, and motivated by this challenge, recent work \cite{zhao2021calibrating,roth2024forecasting} have conjectured that multi-class calibration requires exponential time. 
In response, a variety of alternative notions have been proposed \cite{zhao2021calibrating, noarov2023high, kleinberg2023u, gopalan2024computationally} to sidestep this intractability while retaining useful properties of calibration.





\subsection{Our results}
In this work, we revisit the problem of high dimensional online calibration and show that for any fixed accuracy parameter $\eps > 0$, there exists a randomized forecasting strategy that achieves $\eps$-calibration in polynomial number of rounds.

\begin{restatable}[]{theorem}{calibrationAlgo}
\label{thm:calibration}
For any $\eps > 0$, there is a randomized forecasting strategy that becomes $\eps$-calibrated after $T = d^{\tilde{O}(1/\eps^2)}$ days.
\end{restatable}
Theorem~\ref{thm:calibration} establishes the first high dimensional forecasting strategy that obtains non-trivial calibration guarantee after {\em polynomial} number of days. 
It works against adaptive adversary and has only $d\log(1/\eps)$ computation cost per day.

Besides its theoretical efficiency, the algorithm is surprisingly simple and interpretable. 
Unlike previous work, which either use a computational inefficient minimax argument \cite{hart2022calibrated,dagan2024breaking}, or apply no-swap regret learning over an exponentially large $\eps$-net \cite{foster1997calibrated}; our forecaster randomly selects from a collection of $\log(d)/\varepsilon^2$ sub-forecasters, where each sub-forecaster simply outputs the the empirical outcome frequency over recent time window. 
We prove this simple strategy obtains vanishing calibration error in polynomial iterations!

While our algorithm achieves polynomial dependence on the dimension $d$, it incurs an exponential dependence on $1/\varepsilon$. To understand this limitation, we complement our algorithm with a lower bound

\begin{restatable}[]{theorem}{calibrationLower}
\label{thm:lower}
For any $\eps \in (2^{-d^{1/3}}, 1)$, no algorithm can guarantee $\eps$-calibration in fewer than $T = d^{\tilde{O}(\log(1/\eps))}$ rounds.
\end{restatable}

We note the lower bound in Theorem \ref{thm:lower} does not match our algorithm in Theorem \ref{thm:calibration}, closing this gap is left as an open question. Nevertheless, the lower bound has several important implications. 
First, it implies that a polynomial number of iterations (i.e., $d^{\Omega(1)}$) are required even for constant $\eps > 0$; if one wants to go further and set $\eps = 1/\poly(d)$, then one needs super-polynomial number of iterations. 
It also shows that no algorithm could guarantee $\poly(d)\cdot T^{1-\delta}$ calibration error for some absolute constant $\delta > 0$ that is independent of dimension $d$, and therefore, establishes a separation between binary prediction ($d=2$) and high dimensional prediction.

Together, Theorems~\ref{thm:calibration} and \ref{thm:lower} imply that for constant $\eps > 0$, $\poly(d)$ rounds are both necessary and sufficient to achieve $\eps$-calibration; for high accuracy $\eps =1/\poly(d)$, super-polynomial iterations are necessary. 
These results (partially) resolve the long-standing open questions of \cite{abernethy2011does, hazan2012weak}, and open new avenues for practical and theoretical advances in high-dimensional calibration.

\subsection{Related work}
\paragraph{Online calibration}
There is a long line of work on online (sequential) calibration \cite{dawid1982well,foster1997calibrated,foster1998asymptotic,qiao2021stronger,dagan2024breaking,hart2022calibrated,foster1999proof, fudenberg1999easier,kakade2008deterministic, mannor2007online, mannor2010geometric,abernethy2011does,hazan2012weak,foster2018smooth,luo2024optimal, noarov2023high,kleinberg2023u,garg2024oracle,qiao2024distance,arunachaleswaran2025elementary}.
\cite{foster1998asymptotic} give the first calibrated forecasting algorithm over binary outcome, using Brier score and no swap regret learning. 
The same approach was later extended to multi-class calibration \cite{foster1997calibrated}, but requires $(1/\eps)^{\Omega(d)}$ iterations to be $\eps$-calibrated.
There are several alternative approaches for calibrated forecasting \cite{hart2022calibrated,foster1999proof, mannor2007online,mannor2010geometric} using minimax argument \cite{hart2022calibrated} or Blackwell's approachability \cite{foster1999proof}. 
While these classical work give a variety of asymptotically calibrated algorithms, the (optimal) statistical efficiency remains unclear.

For binary outcomes, it has long been known that the optimal total calibration error lies within the range $[T^{1/2}, T^{2/3}]$, with improvements made only recently. 
On the lower bound side, \cite{qiao2021stronger} prove a lower bound of $\Omega(T^{0.528})$ total calibration error when the forecaster faces an adaptive adversary; \cite{dagan2024breaking} strengthens the lower bound to $\Omega(T^{0.543})$ and it holds even against an oblivious adversary.
On the algorithmic side, the recent breakthrough \cite{dagan2024breaking} give the first algorithm with $O(T^{2/3-\delta})$ total calibration error for some constant $\delta > 0$. 

Despite the recent progress on binary calibration, high dimensional calibration remains challenging. The best known algorithms \cite{foster1997calibrated,hart2022calibrated, fishelson2025full} take $(1/\eps)^{\Omega(d)}$ iterations.
\cite{abernethy2011does} pose a COLT open question on the computational efficiency of high dimensional online calibration. 
\cite{hazan2012weak} prove a computational lower bound for deterministic weak calibration forecaster, in particular, assuming $\mathsf{PPAD} \subseteq \mathsf{RP}$, there is no polynomial time deterministic calibrated forecaster that could be $\eps = 1/d^3$-calibrated.
\cite{hazan2012weak} pose the statistical complexity of (randomized) online calibration as an open question in the discussion section.


\paragraph{The benefits of calibration} Beyond being a desirable property in its own right, calibration has proven valuable for downstream decision-making tasks, including swap regret minimization \cite{kleinberg2023u, roth2024forecasting, hu2024predict}, equilibrium computation in games \cite{foster1997calibrated, haghtalab2023calibrated}, and fairness considerations \cite{pleiss2017fairness, hebert2018multicalibration}.

\paragraph{Calibration in other setting} 
There is a vast body of literature on calibration across various areas, including fairness \cite{pleiss2017fairness, hebert2018multicalibration}, machine learning \cite{guo2017calibration, braverman2020calibration, minderer2021revisiting, kalai2024calibrated} and medical care \cite{jiang2012calibrating, crowson2016assessing}, see the reference therein.

\subsection{Technique overview}
We give a high level overview on the technical approach of Theorem \ref{thm:calibration} and Theorem \ref{thm:lower}. 
Section \ref{sec:tech-algo} presents the calibrated forecasting algorithm and its analysis;
Section \ref{sec:tech-intuition} explains the intuition behind the algorithm and how we build upon the previous work of \cite{foster1997calibrated,peng2023fast,dagan2024external}.
Section \ref{sec:tech-lower} discusses the ideas for lower bounds.

\subsubsection{The forecasting algorithm and its analysis}
\label{sec:tech-algo}
Our forecasting algorithm maintains multiple sub-forecasters, each operating at different scales of granularity. Specifically, the $\ell$-th sub-forecaster partitions the prediction sequence into $H^{\ell-1}$ ($H = 1/\eps^4$) intervals, each of length $T/H^{\ell-1}$. Within each interval, it uses the empirical outcome frequency as the prediction and updates every $T/H^{\ell}$ days (so it updates $H$ times within each interval).
On each day, the final forecast is sampled uniformly at random from these $L = \log(d)\cdot \eps^{-2}$ sub-forecasters.

At a first glance, using empirical outcome frequency as a prediction might be a bad idea. For example, consider the first forecaster, it only has one interval and it updates every $T/H$ days. Let $X_{1}, \ldots, X_{H}$ be the empirical frequency of days $[1:T/H], \ldots, [(H-1)\cdot(T/H)+1: T]$. The average calibration error of the first forecaster equals 
\[
\frac{1}{H}\sum_{h\in [H]} \Big\|\frac{X_{1}+\cdots +X_{h-1}}{h-1} - X_h\Big\|_1.
\]
This value could be large, for example, if $X_1,\ldots, X_h$ spread out like $X_1 = (1,0, \ldots, 0), X_{2} = (0,1,0, \ldots, 0), \ldots$, then using empirical frequency seems to be a bad idea, as the outcome at the $h$-th step could be very different from the historical outcome. 
Nevertheless, {\em our crucial observation is that, whenever this happens, the average entropy of $\Ent(X_1),\ldots, \Ent(X_{H})$ must be smaller than the entropy $\Ent(\frac{X_1+\cdots +X_{H}}{H})$ by a non-trivial amount.} 
This motivates one to further divide $X_1, \ldots, X_{H}$ into finer granularity as the entropy can not drop forever.
Overall, we prove that by averaging across different sub-forecasters, the average entropy drop scales like $\log(d)/L$ and the average calibration error is $\sqrt{\log(d)/L}$.

\subsubsection{The intuition behind the algorithm} 
\label{sec:tech-intuition}
The purpose of this section is to explain the origin of our new forecasting algorithm. The algorithm and its analysis have already been sketched in Section \ref{sec:tech-algo}, so readers could skip this section if they are not interested in how we design the forecasting algorithm and how it relates with existing literature.
As we shall explain, our algorithm combines the classic approach of \cite{foster1998asymptotic} with the recently developed faster no-swap regret learning algorithm \cite{peng2023fast,dagan2024external}, in addition with a few new ideas.

We first review the approach of \cite{foster1998asymptotic}. In a nutshell, \cite{foster1998asymptotic} first reduce online calibration to swap regret minimization, then use a swap regret minimization algorithm \cite{foster1998asymptotic, blum2007external}.
In more details, one unique challenge for online calibration is that the calibration error is not additively separable: one can not directly attribute the $\ell_1$ error to each individual day $t\in [T]$.
To get around with it, \cite{foster1998asymptotic} consider a surrogate loss, a.k.a. the Brier score $\|p_t - X_t\|_2^2$. 
The Brier score is additive and \cite{foster1998asymptotic} prove that one can reduce online calibration to swap regret minimization on Brier score.
In particular, if the algorithm has at most $\delta$-swap regret, then it is $\eps = \sqrt{d\delta}$-calibrated.
The swap regret is minimized w.r.t. the $\eps$-net of $\Delta_d$ (denoted as $\mathcal{N}_{\eps}$), and a classical result of \cite{foster1998asymptotic,blum2007external} shows that one can obtain $\delta$-swap regret in $T = |\mathcal{N}_{\eps}|/\delta^2 \approx (1/\eps)^{\Theta(d)}$ days, this is where the exponential dependence on $d$ comes from.


The no-swap regret learning algorithm of \cite{blum2007external} has long been considered optimal.
Nevertheless, the recent work of \cite{peng2023fast,dagan2024external} give an alternative algorithm, which obtain $\delta$-swap regret after $(\log(|\mathcal{N}_{\eps}|))^{1/\delta}$ days. 
With this new algorithm, we can hope to improve the classic approach of \cite{foster1998asymptotic}. 
Nevertheless, there are several challenges to overcome, which are the new technical contribution of this paper.
\begin{itemize}
    \item {\bf Better reduction via cross entropy.} The first issue comes from the error blowup in the calibration-to-swap-regret reduction. To obtain $\eps$-calibration, one needs the swap regret to be $\delta = \eps^2/d$. 
    This extra factor of $d$ is critical, as the swap regret algorithm of \cite{peng2023fast,dagan2024external} now requires $T = (\log(|\mathcal{N}_{\eps}|))^{1/\delta} \approx d^{d/\eps^2}$, which is actually worse than the original approach \cite{foster1998asymptotic}. To this end, we give a more efficient reduction and use the cross entropy loss in replace of the Brier score, i.e., the surrogate loss we minimize is $\langle X_t, \log(1/p_t)\rangle$ instead of $\|X_t - p_t\|_2^2$.
    The cross entropy loss gives a better reduction, to obtain $\eps$-calibration, one only needs to obtain $\delta = \eps^2/\log(d)$-swap regret.
    \item {\bf ``Purify'' prediction via a new no-external regret algorithm.} Using the cross entropy loss, one could hope to obtain a quasi-polynomial forecaster with $T = (\log(|\mathcal{N}_{\eps}|))^{1/\delta} \approx d^{\log(d)/\eps^2}$. However, there is a very subtle issue: The no-swap regret learning algorithm of \cite{peng2023fast,dagan2024external} randomizes over multiple no-external regret algorithms, where each no-external regret algorithm commits a distribution over $\mathcal{N}_{\eps}$, but not a single prediction from $\mathcal{N}_{\eps}$. 
    This is problematic since (1) the forecaster must make a prediction rather than commit a distribution over $\mathcal{N}_{\eps}$;\footnote{Random sampling would not work here, it introduces huge variance when $T = 2^{o(d)}$} (2) for a generic no-swap regret algorithm, it must commit a distribution (rather than a single action) and this is very critical for the new algorithm of \cite{peng2023fast,dagan2024external} (see their paper for discussion). 
    To this end, we observe that our regret minimization task has certain nice structural property, one can design a new external regret algorithm that commits a single prediction from $\mathcal{N}_{\eps}$ instead of a distribution over $\mathcal{N}_{\eps}$. In fact, this prediction is the empirical frequency of the outcome! We prove this prediction has no-external regret property, by the reduction in \cite{peng2023fast,dagan2024external}, this would translate to no-swap regret property.
    \item {\bf Parameter optimization via smoothness.} We finally remove the $\log(d)$ factor from the exponent and design a polynomial time forecaster. 
    The observation is that the new external regret algorithm is very smooth, its prediction within any $\eps d$ iterations are $\eps$-close in $\ell_1$ distance. One can use a lazy update strategy and further improve the no-swap regret algorithm from $T = (\log(|\mathcal{N}_{\eps}|))^{1/\delta} = d^{\log(d)/\eps^2}$ to  $T = (1/\eps^4)^{1/\delta} = d^{1/\eps^2}$ (notably this means that our no-external regret algorithm does not even have a logarithmic dependence on $|\mathcal{N}_{\eps}|$).
\end{itemize}

\subsubsection{Lower bound}
\label{sec:tech-lower}
We next sketch our lower bound construction.
Our lower bound is also inspired by the swap regret learning lower bound of \cite{peng2023fast,dagan2024external} and has a recursive structure. However, online calibration also differs from swap regret minimization as one can not arbitrarily penalize a set of ``wrong decisions''. 
Our high level idea is to enforce the forecaster to predict distinctively across different days -- if the set of predictions (conditional events) are huge, then the calibration error must also be large. Nevertheless, this is not easy because hedging strategy exists (as shown by our forecasting algorithm).

Our lower bound works even in the setting where the forecaster knows the outcome distribution $q_t \in \Delta_d$ (but not the actual outcome $X_t \sim q_t, X_t \in [d]$) at the beginning of day $t$. 
The lower bound has a recursively structure. Let $R = \log(1/\eps)$ be the total number of recursions and we partition the outcome space into $R$ blocks $[d] = D_1\cup \cdots \cup D_{R}$. 
We first divide the sequence $[T]$ into $K \ll R$ intervals $T_1, \ldots, T_{K}$, each of length $T/K$.
For each $k\in [K]$, let $i_k$ be selected uniformly at random from $D_1$. 
During time interval $T_k$, the outcome distribution $q_t$ ($t \in T_k$) keeps the same for $D_1$ while varies for $D_2$.
In particular, the distribution $q_t$ always has $1/R$ probability mass on $i_k$ and $0$ probability mass on $D_{1}\setminus \{i_k\}$,
The distribution over the rest outcome $D_2,\ldots, D_{R}$ varies during $T_{k}$ and they are constructed recursively in the same way.

For the analysis, if the forecaster is truthful, in the sense that it always truthfully predicts the distribution over $D_1$ during $T_{1}, \ldots, T_{K}$ (it could be non-truthful over other outcome $D_2\cup\cdots D_{k}$), then the prediction set is disjoint for each $k\in [K]$ and the calibration error is additive.
However, the forecaster needs not to be truthful over $D_1$, for example, during time interval $T_{k}$, it could deliberately make some repeated prediction $p_t$ that have been made in previous time interval $T_{k'}$ ($k'<k$), with the hope of balancing the empirical outcome and reducing the calibration error.
The critical observation is that, $i_{k}$ is chosen randomly and not known to the forecaster at $T_{k'}$, it is very unlikely that the forecaster could guess the appearance of $i_k$ (recall that $K\ll R$), and therefore, the time that the forecaster deliberately balances the empirical distribution over $D_{2}\cup \cdots \cup D_{R}$, it must incur roughly $1/R$ error on $D_1$. 
Consequently, the average error decays by at most  a constant factor ($1/R$) for each recursion, there are $R$ recursion so the final average error is $R^{-R}\approx \eps$ and the total number of days are $K^{R} \approx d^{\log(1/\eps)}$.

\section{Preliminary}

\paragraph{Notation} Throughout the paper, we write $[n] = \{1,2, \ldots, n\}$, and $[n_1:n_2] = \{n_1, n_1+1, \ldots, n_2\}$. For any set $S$, we use $\Delta(S)$ to denote all probability distributions over $S$, and for any integer $d$, we write $\Delta_d = \Delta([d])$ for simplicity. We write $\|x\|_1 = \sum_{i}|x_i|$ to denote the $\ell_1$ norm of a vector $x$. 
For a random variable $X$, we write $\Ent(X)$ to denote the entropy of $X$ in natural log base, i.e., $\Ent(X) = \sum_{x}\Pr[X= x]\log(1/\Pr(X=x))$. We use $\KL(X||Y)$ to denote the KL divergence between $X$ and $Y$.

\paragraph{Calibrated forecasting} In the task of online forecasting, there is a set of outcome $[d]$ and a forecaster makes prediction over a sequence of $T$ days.
It is common to discretize the prediction space and consider the prediction coming from a finite set $\mK \subseteq \Delta_d$.
At each day $t \in [T]$, the forecaster first makes a prediction $p_t \in \mK\subseteq \Delta_d$ over the outcome space $[d]$, then the Nature reveals the outcome $X_t \in [d]$. 
It is well-known that, in order to achieve non-trivial calibration guarantee, the prediction of a forecaster must be randomized \cite{oakes1985self,foster1998asymptotic}. 
Hence, at each day, the prediction $p_t$ is drawn from some distribution $\mu_t \in \Delta(\mK)$ over the prediction space $\mK$.

The calibration measures the conditional accuracy of the online forecaster. The most commonly used metric is the $\ell_1$-calibration.
\begin{definition}[$\ell_1$ calibration]
Given a forecaster with prediction drawn from $\mu_1, \ldots, \mu_{T} \in \Delta(K)$ and outcome $X_1, \ldots, X_T$, the expected calibration error (ECE) is defined as 
\begin{align}
\mathsf{ECE}_{\mu, X}:= \E_{\{p_t \sim \mu_t\}_{t\in [T]}}\left[\sum_{p \in \mK} \Big\| \sum_{t=1}^{T} (p_t - X_t) \cdot \mathsf{1}[p_t = p]\Big\|_1\right] \label{eq:ece}
\end{align}
Here we view $X_t$ as a dimension $d$ one-hot vector with the $X_t$-th coordinate equals $1$.
\end{definition}

Since the prediction $p_t$ is drawn from the distribution $\mu_t$, one could also define the calibration error with respect to the distribution -- this is called the distributional calibration in \cite{foster1998asymptotic}.
\begin{definition}[$\ell_1$ distributional calibration]
Given a calibrated forecaster with prediction distribution $\mu_1, \ldots, \mu_{T} \in \Delta(K)$ and outcome $X_1, \ldots, X_T$, the distributional calibration error (DCE) is defined as 
\begin{align}
\mathsf{DCE}_{\mu, X}:= \sum_{p \in \mK} \Big\| \sum_{t=1}^{T}  (p - X_t) \cdot \mu_t(p)\Big\|_1 \label{eq:dce}
\end{align}
\end{definition}

We say a forecasting algorithm is $\eps$-calibrated (resp. $\eps$-distributional calibrated), if its expected calibration error (resp. distributional calibration error) is at most $\eps T$.
We note that an $\eps$-calibrated forecaster implies an $\eps$-distributional calibrated forecaster, while the other direction does not necessarily hold.



\paragraph{Adversary model} 
In online forecasting, an adaptive adversary refers the Nature could adaptively choose the outcome $X_t$ based on the past prediction $p_1,\ldots, p_{t-1}$. An oblivious adversary refers the Nature would (randomly) choose the outcome sequence $X_1,\ldots, X_{T}$ at the beginning of day $1$. Our algorithm works for adaptive adversary while the lower bound holds against oblivious adversary.




\section{Calibrated forecasting in polynomial rounds}

\calibrationAlgo*


\paragraph{Algorithm description}
Our approach is depicted in Algorithm \ref{algo:forecaster}. At each day $t$, the prediction is obtained by randomly sampling from $L = \log(n)/\eps^2$ sub-forecasters (Line \ref{line:sample}). For each $\textsc{Forecaster}(\ell)$ ($\ell\in [L]$), it divides the entire sequence into $H^{\ell-1}$ intervals of equal size. Roughly speaking, within each interval, $\textsc{Forecaster}(\ell)$ uses the the empirical frequency of the outcome as the prediction.
More precisely, at interval $h_{<\ell} = (h_1, \ldots, h_{\ell-1}) \in [H]^{\ell-1}$, $\textsc{Forecaster}(\ell)$ operates in $H$ iterations, where each iteration contains $T_{\ell} = T/H^{\ell}$ consecutive days. 
It starts with an uniform distribution $\vec{\mathsf{1}}_{d} = \frac{1}{d}(1,1, \ldots, 1)$, and for every $T_{\ell}$ days, it computes the empirical frequency of the outcome within the $h_{<\ell}$-th interval up to this point (Eq.~\eqref{eq:empirical-outcome}), and then predicts it for the next $T_{\ell}$ days (Line \ref{line:predict}).

\begin{algorithm}[!hbtp]
\caption{Calibrated forecaster}
\label{algo:forecaster}
\begin{algorithmic}[1]
\State {\bf Parameters} $L = \log(n)/\eps^2, H = 1/\eps^4, T = (d^3/\eps^6)\cdot H^{L}$
\State {\bf Parameters} $T_{\ell} = (d^3/\eps^6) \cdot H^{L - \ell}$ ($\ell \in [L]$), 
\State {\bf Parameters} $\Gamma_{h_{1}, \ldots, h_{\ell}} := [\sum_{r =1}^{\ell}(h_{r}-1)T_{r}+1: \sum_{r =1}^{\ell}(h_{r}-1)T_{r}+ T_{\ell}]$\\
\For{$t =1,2,\ldots, T$}
\State Random sample $\ell \in [L]$ and make the prediction $p_t^{(\ell)} \in \Delta_{d}$ from $\textsc{Forecaster}(\ell)$\label{line:sample}
\EndFor\\
\Procedure{\textsc{Forecaster}}{$\ell$} \Comment{$\ell \in [L]$}
\For{$h_{<\ell} = 1,2, \ldots, H^{\ell-1}$} 
\For{$h_{\ell} = 1,2,\ldots, H$}
\State Compute the average outcome  \Comment{$\vec{\mathsf{1}}_{d} =\unif([d])$} 
\begin{align}
Y^{(\ell)}_{h_{<\ell}, h_{\ell}} = \frac{1}{(h_{\ell}-1 + 1/\eps)T_{\ell}}\left(\sum_{h < h_{\ell}}\sum_{t \in \Gamma_{h_{<\ell}, h}}X_t + (T_{\ell}/\eps) \cdot \vec{\mathsf{1}}_{d} \right)  \label{eq:empirical-outcome}
\end{align}
\State Predict $p^{(\ell)}_{t} = Y^{(\ell)}_{h_{<\ell}, h_{\ell}}$ for the next $T_{\ell}$ days \Comment{$p^{(\ell)}_{t} = Y^{(\ell)}_{h_{<\ell}, h_{\ell}}$ for $t \in \Gamma_{h_{<\ell},h_{\ell}}$} \label{line:predict}
\EndFor
\EndFor
\EndProcedure
\end{algorithmic}
\end{algorithm}

\paragraph{Notation}
We use the notation $h_{<\ell} = (h_{1}, \ldots, h_{\ell-1})$ and $h_{\leq \ell} = (h_{1}, \ldots, h_{\ell})$ interchangeably.
For any $h_{\leq \ell}\in [H]^{\ell}$, we write $\Gamma_{h_{\leq\ell}} := [\sum_{r =1}^{\ell}(h_{r}-1)T_{r}+1: \sum_{r =1}^{\ell}(h_{r}-1)T_{r}+ T_{\ell}]$. We note $\Gamma_{h_{\leq \ell}}$ is the $h_{\leq \ell} = (h_{1}, \ldots, h_{\ell})$-th interval of $\textsc{Forecaster}(\ell+1)$, and it is also the $h_{\ell}$-th iteration in the $h_{<\ell}$-th interval of $\textsc{Forecaster}(\ell)$. Let $X_{h_{\leq\ell}} \in \Delta_{d}$ be the average outcome in $\Gamma_{h_{\leq\ell}}$, i.e., 
$X_{h_{\leq \ell}} = \frac{1}{T_{\ell}} \sum_{t\in \Gamma_{h_{\leq\ell}}}X_{t}$.
For each $\ell \in [L]$, we note that $p_{t}^{(\ell)}$ remains the same for $t\in \Gamma_{h_{\ell}}$, hence we can write it as $p_{h_{\leq \ell}}^{(\ell)}$.

\paragraph{Distributional calibration guarantee}
The key step is to derive the distributional calibration guarantee of Algorithm \ref{algo:forecaster}. 
Let $q_{t}$ be the distribution of the prediction at the $t$-th day, i.e., 
\begin{align}
q_t(p) = \frac{1}{L}\sum_{\ell=1}^{L}\mathsf{1}[p_{t}^{(\ell)} = p] \quad \quad \forall p \in \Delta_d. \label{eq:q_t}
\end{align}
We prove that $\{q_t\}_{t\in [T]}$ is $4\eps$-distributional calibrated.
\begin{lemma}[Distributional calibration guarantee]
\label{lem:distributional}
Algorithm \ref{algo:forecaster} is an $4\eps$-distributional calibrated forecaster, i.e., 
\begin{align*}
\DCE_{q, X} \leq 4\eps T.
\end{align*}
\end{lemma}

\begin{proof} Our proof proceeds in a few steps.

{\noindent \bf Step 1.} We first attribute the distributional calibration error to different levels $\ell \in [L]$ and intervals $h_{\leq \ell} \in [H]^{\ell}$
\begin{align}
\DCE_{q, X}= &~ \sum_{p} \Big\| \sum_{t=1}^{T}  (p - X_t) \cdot q_t(p)\Big\|_1 \notag \\
= &~ \frac{1}{L}\sum_{p} \Big\| \sum_{t=1}^{T}  (p - X_t) \cdot \sum_{\ell=1}^{L}\mathsf{1}[p_t^{(\ell)}=p]\Big\|_1 \notag\\
\leq &~ \frac{1}{L}\sum_{p}\sum_{\ell=1}^{L}\sum_{h_{\leq \ell} \in [H]^{\ell}}\Big\| \sum_{t\in \Gamma_{h_{\leq \ell}}}  (p - X_t) \cdot\mathsf{1}[p_t^{(\ell)}=p]\Big\|_1\notag \\
= &~ \frac{1}{L}\sum_{\ell=1}^{L}\sum_{h_{\leq \ell}\in [H]^{\ell}}  T_{\ell}\cdot \| p_{h_{\leq \ell}}^{(\ell)} - X_{h_{\leq \ell}}\|_1.\label{eq:step1}
\end{align}
The first step follows from the definition of DCE (see Eq.~\eqref{eq:dce}), 
the second step follows from the definition of $q_t$ (see Eq.~\eqref{eq:q_t}).
the third step follows from the triangle inequality. The last step follows from $p_{h_{\leq \ell}}^{(\ell)} = p_t^{(\ell)}$ for any $t \in \Gamma_{h_{\leq \ell}}$, and the definition of $X_{h_{\leq \ell}}$.

{\noindent \bf Step 2.}  Instead of bounding the different between $\| p_{h_{\leq \ell}}^{(\ell)} - X_{h_{\leq \ell}}\|_1$, we bound $\| p_{h_{<\ell}, h_{\ell}+1}^{(\ell)} - X_{h_{\leq \ell}}\|_1$ and note that $p_{h_{<\ell}, h_{\ell}}^{(\ell)}$ is close to $p_{h_{<\ell}, h_{\ell}+1}^{(\ell)}$. In particular, we have  
\begin{lemma}
\label{lem:eps-diff}
For any $\ell \in [L]$, $h_{\leq\ell} \in [H]^{\ell}$, we have
\[
\| p^{(\ell)}_{h_{<\ell}, h_{\ell}} - p^{(\ell)}_{h_{<\ell}, h_{\ell}+1}\|_1 \leq 2\eps.
\]
\end{lemma}
The proof of Lemma \ref{lem:eps-diff} is deferred to the end, now by Eq.~\eqref{eq:step1} and Lemma \ref{lem:eps-diff}, we have that
\begin{align}
\DCE_{q, X} \leq &~ \frac{1}{L}\sum_{\ell=1}^{L}\sum_{h_{\leq \ell}\in [H]^{\ell}}  T_{\ell}\cdot \| p_{h_{\leq \ell}}^{(\ell)} - X_{h_{\leq \ell}}\|_1\notag\\
\leq &~ \frac{1}{L}\sum_{\ell=1}^{L}\sum_{h_{\leq \ell} \in [H]^{\ell}}  T_{\ell}\cdot \| p_{h_{<\ell}, h_{\ell}+1}^{(\ell)} - X_{h_{\leq \ell}}\|_1 + 2\eps T.\label{eq:step2}
\end{align}

{\noindent \bf Step 3.} Instead of bounding the summation of $\ell_1$ distance, we bound the summation of KL divergence, i.e., 
\begin{align}
&~\frac{1}{L}\sum_{\ell=1}^{L}\sum_{h_{\leq \ell} \in [H]^{\ell}}  T_{\ell}\cdot \| p_{h_{<\ell} , h_{\ell}+1}^{(\ell)} - X_{h_{\leq \ell}}\|_1\notag\\
= &~ T\cdot \frac{1}{L}\sum_{\ell=1}^{L} H^{-\ell}\sum_{h_{\leq \ell}\in [H]^{\ell}} \| p_{h_{<\ell}, h_{\ell}+1}^{(\ell)} - X_{h_{\leq \ell}}\|_1 \notag \\
\leq &~ T\cdot \sqrt{\frac{1}{L}\sum_{\ell=1}^{L} H^{-\ell}\sum_{h_{\leq \ell} \in [H]^{\ell}} \| p_{h_{<\ell}, h_{\ell}+1}^{(\ell)} - X_{h_{\leq \ell}}\|_1^2} \notag \\
\leq &~ T\cdot \sqrt{\frac{2}{L}\sum_{\ell=1}^{L} H^{-\ell}\sum_{h_{\leq\ell} \in [H]^{\ell}} \KL(X_{h_{\leq \ell}} || p_{h_{<\ell}, h_{\ell}+1}^{(\ell)})}\label{eq:step3}.
\end{align}
The first step follows from $T_{\ell} = T/H^{\ell}$. The second step follows from the Cauchy Schwarz inequality and the third step follows from the Pinsker inequality.

{\noindent \bf Step 4.}
It remains to bound the summation of KL divergence. First, by the definition of KL divergence, we have that 
\begin{align}
&~ \frac{1}{L}\sum_{\ell=1}^{L} H^{-\ell}\sum_{h_{\leq\ell} \in [H]^{\ell}} \KL(X_{h_{\leq\ell}} || p_{h_{<\ell}, h_{\ell}+1}^{(\ell)})\notag\\
= &~ \frac{1}{L}\sum_{\ell=1}^{L} H^{-\ell}\sum_{h_{\leq \ell} \in [H]^{\ell}} \Big(\left\langle X_{h_{\leq \ell}}, \log(1/p_{h_{<\ell}, h_{\ell}+1}^{(\ell)})\right\rangle - \Ent(X_{h_{\leq\ell}})\Big).\label{eq:step4-1}
\end{align}

Our crucial observation is 
\begin{lemma}
\label{lem:pseudo-regret}
For any $h_{<\ell}\in [H]^{\ell-1}$, we have 
\begin{align*}
\sum_{h_{\ell}\in [H]}\langle X_{h_{\leq \ell}}, \log(1/p_{h_{<\ell},h_{\ell}+1}^{(\ell)})\rangle \leq H\cdot  \mathsf{Ent}(X_{h_{<\ell}}) + \eps^2 H.
\end{align*}
\end{lemma}

The proof of Lemma \ref{lem:pseudo-regret} is deferred to the end.
Now, we can telescope the summation, and we have
\begin{align}
&~ \frac{1}{L}\sum_{\ell=1}^{L} H^{-\ell}\sum_{h_{\leq \ell} \in [H]^{\ell}} \Big(\left\langle X_{h_{\leq \ell}}, \log(1/p_{h_{<\ell}, h_{\ell}+1}^{(\ell)})\right\rangle - \Ent(X_{h_{\leq\ell}})\Big)\notag\\
\leq &~  \frac{1}{L}\sum_{\ell=1}^{L} H^{-\ell}\sum_{h_{<\ell} \in [H]^{\ell-1}}  \Big(H\cdot \Ent(X_{h_{<\ell}}) +\eps^2 H  - \sum_{h_{\ell}\in [H]}\Ent(X_{h_{\leq \ell}})\Big)\notag\\
= &~ \frac{1}{L} \Ent(X) - \frac{1}{L} H^{-L}\sum_{h_{\leq L}\in [H]^{L}}\Ent(X_{h_{\leq L}}) + \eps^2\notag \\
\leq &~ \frac{\log(d)}{L} + \eps^2 = 2\eps^2 \label{eq:step4-2}.
\end{align}
The first step follows from Lemma \ref{lem:pseudo-regret}. The second step takes the telescoping sum, we slightly abuse of notation and write $X = \frac{1}{H}\sum_{h_1\in [H]}X_{h_1}$. The third step holds since $0 \leq \Ent(Y) \leq \log(d)$ for any random variable over $[d]$ and the last step holds since we take $L =\log(d)/\eps^2$.

Combining Eq.~\eqref{eq:step4-1}\eqref{eq:step4-2}, we have
\begin{align}
\frac{1}{L}\sum_{\ell=1}^{L} H^{-\ell}\sum_{h_{\leq\ell} \in [H]^{\ell}} \KL(X_{h_{\leq\ell}} || p_{h_{<\ell}, h_{\ell}+1}^{(\ell)}) \leq 2\eps^2. \label{eq:step4}
\end{align}

Finally, combining Eq.~\eqref{eq:step2}\eqref{eq:step3}\eqref{eq:step4}, we have
\begin{align*}
\DCE_{q, X} \leq 2\eps T + T\cdot \sqrt{2 \cdot 2\eps^2} = 4\eps T.
\end{align*}

\end{proof}

It remains to complete the proof of Lemma \ref{lem:eps-diff} and Lemma \ref{lem:pseudo-regret}.

\begin{proof}[Proof of Lemma \ref{lem:eps-diff}]
For any $\ell \in [L]$, $h_{\leq\ell}\in [H]^{\ell}$, by definition (i.e., Eq.~\eqref{eq:empirical-outcome}), one has 
\begin{align*}
p^{(\ell)}_{h_{<\ell}, h_{\ell}} = \frac{\sum_{h < h_{\ell}}X_{h_{<\ell}, h} + (1/\eps )\cdot \vec{1}_{d}}{h_{\ell}-1+1/\eps}
\end{align*}
and 
\[
p^{(\ell)}_{h_{<\ell}, h_{\ell}+1} = \frac{\sum_{h \leq h_{\ell}}X_{h_{<\ell}, h} + (1/\eps ) \cdot \vec{1}_{d}}{h_{\ell}+1/\eps} = \frac{h_{\ell}-1 + 1/\eps}{h_{\ell} +1/\eps} p^{(\ell)}_{h_{<\ell}, h_{\ell}} + \frac{1}{h_{\ell} +1/\eps} X_{h_{<\ell}, h_{\ell}}.
\]
Their difference can be bounded as
\begin{align*}
\| p^{(\ell)}_{h_{<\ell}, h_{\ell}} - p^{(\ell)}_{h_{<\ell}, h_{\ell}+1}\|_1 \leq &~ \frac{1}{h_{\ell}+1/\eps} \|p_{h_{<\ell}, h_{\ell}}^{(\ell)}\|_1 + \frac{1}{h_{\ell}+1/\eps}\|X_{h_{<\ell}, h_{\ell}}\|_1 \\
\leq &~ \frac{1}{h+1/\eps} + \frac{1}{h+1/\eps} \leq 2\eps.
\end{align*}
\end{proof}

\begin{proof}[Proof of Lemma \ref{lem:pseudo-regret}]
We fix $h_{<\ell}\in [H]^{\ell-1}$ in the proof. 
For any $h_{\ell}\in [H]$, recall the definition of $p_{h_{<\ell}, h_{\ell}}^{(\ell)} \in \Delta_d$ (see Eq.~\eqref{eq:empirical-outcome})
\begin{align*}
p^{(\ell)}_{h_{<\ell}, h_{\ell}} = \frac{\sum_{h < h_{\ell}}X_{h_{<\ell}, h} + (1/\eps )\cdot \vec{1}_{d}}{h_{\ell}-1+1/\eps}.
\end{align*}
We simplify the notation a bit and write $w_{h} = X_{h_{<\ell}, h} \in \Delta_{d}$ ($h \in [H]$) and $z_h = p_{h_{<\ell}, h} \in \Delta_d$ ($h \in [H+1]$). Then we have
\begin{align*}
z_h(i) = \frac{\sum_{\tau < h}w_{\tau}(i) + 1/\eps d}{h-1+1/\eps} \quad \quad \forall i \in [d].
\end{align*}

Now we have 
\begin{align}
\sum_{h_{\ell}\in [H]}\left\langle X_{h_{\leq\ell}}, \log(1/p_{h_{<\ell}, h_{\ell}+1}^{(\ell)})\right\rangle= &~ \sum_{h\in [H]}\left\langle w_h, \log(1/z_{h+1})\right\rangle\notag\\
= &~ \sum_{h\in [H]}\sum_{i\in [d]} w_h(i) \log\left(\frac{h+1/\eps}{\sum_{\tau\leq h}w_\tau(i) + 1/\eps d}\right)\notag\\
= &~ \sum_{h=1}^{H}\log(h+1/\eps) +  \sum_{h\in [H]}\sum_{i\in [d]} w_h(i) \log\left(\frac{1}{\sum_{\tau \leq h}w_\tau(i)+ 1/\eps d}\right).\label{eq:regret1}
\end{align}
The first two steps follow from the definition of $w_h$ and $z_{h+1}$ and the last step follows from $\sum_{i\in[d]}w_{h}(i) = 1$.

For the first term in the RHS of Eq.~\eqref{eq:regret1}, we have
\begin{align}
 \sum_{h=1}^{H}\log(h+1/\eps) \leq &~ \int_{h=0}^{H}\log(h+1+1/\eps)\mathsf{d}h \notag\\
 =&~ (H+1+1/\eps)\log(H+1+1/\eps)- (1/\eps+1)\log(1/\eps+1) - H\notag\\
 =&~ H\log(H) + H\log(1+\frac{1+1/\eps}{H})  + (1/\eps+1)\log(1+ \frac{H}{1/\eps+1}) - H \notag\\
 \leq &~  H\log(H) - H + \eps^2 H\label{eq:regret2}
\end{align}
Here the second step follows from the rule of integral, the last step follows from 
\[
H\log(1+\frac{1+1/\eps}{H})  + (1/\eps+1)\log(1+ \frac{H}{1/\eps+1})\leq (1+1/\eps)(1+\log(H+1)) \leq \eps^2 H.
\]

For the second term in the RHS of Eq.~\eqref{eq:regret1}, for each $i \in [d]$, let $W_i = \sum_{h\in [H]}w_{h}(i)$, then we have
\begin{align}
 \sum_{h\in [H]} w_h(i) \log\left(\frac{1}{\sum_{\tau \leq h}w_\tau(i) +1/\eps d}\right) \leq &~ -\int_{w=0}^{W_i} \log(w+1/\eps d)\mathsf{d} w\notag\\
 = &~ -(W_i+1/\eps d)\log(W_i+1/\eps d) + (1/\eps d)\log(1/\eps d) + W_i\notag\\
 \leq &~ - W_i \log(W_i)  + W_i\label{eq:regret3}.
\end{align}
The second step follows from the rule of integral.

Combining Eq.~\eqref{eq:regret1}\eqref{eq:regret2}\eqref{eq:regret3}, we get
\begin{align*}
\sum_{h_{\ell}\in [H]}\left\langle X_{h_{\leq\ell}}, \log(1/p_{h_{<\ell}, h_{\ell}+1}^{(\ell)})\right\rangle \leq &~  \left(H\log(H) - H + \eps^2 H)\right) + \left(\sum_{i\in [d]}-W_i \log(W_i) + W_i\right)\\
= &~ H\log(H) -\sum_{i\in [d]}W_i\log(W_i) + \eps^2 H \\
= &~ H \cdot \Ent(X_{h_{<\ell}}) + \eps^2 H.
\end{align*}
Here the second step holds since $\sum_{i\in [d]}W_i = H$, the last step follows from 
$W_i = \sum_{h\in [H]}w_h(i) = \sum_{h\in [H]}X_{h_{<\ell}, h}(i) = H \cdot X_{h_{<\ell}}(i)$.
This completes the proof.
\end{proof}

\paragraph{Expected calibration error} Finally, we bound the expected calibration error of Algorithm \ref{algo:forecaster}. The following Lemma states that the calibration error concentrates within each time interval $\Gamma_{h_{\leq L}}$.
\begin{lemma}
\label{lem:sample-concentration}
    For any $h_{\leq L} \in [H]^{L}$, one has
    \begin{align*}
    \E\left[\sum_{p\in P} \Big\|\sum_{t\in \Gamma_{h_{\leq L}}} (p-X_{t}) \cdot \mathsf{1}[p_t = p] - \sum_{t\in \Gamma_{h_{\leq L}}} (p-X_{t}) \cdot q_{t}(p)\Big\|_1\right] \leq \eps T_{L}
    \end{align*}
\end{lemma}
\begin{proof}
Fix any possible past outcome $\{X_t\}_{t \leq \sum_{\ell\in [L]}(h_{\ell}-1)T_{\ell}}$, condition these past outcome, the prediction $p_{t}^{(\ell)}$ ($t \in \Gamma_{h_{\leq L}}$) are fixed for each forecaster $\ell\in [L]$. 
Define $P_{h_{\leq L}} := \{p_{h_{\leq \ell}}^{(\ell)}\}_{\ell\in [L]}$, for any $p \in P_{h_{\leq L}} $,
and for any $t \in \Gamma_{h_{\leq L}}$, we have that 
\[
\E \Big[(p-X_{t}) \cdot \mathsf{1}[p_t = p] \, |\, \{p_{\tau}, X_{\tau}\}_{\tau \in \Gamma_{h_{\leq L}}, \tau < t}\Big] = (p-X_{t}) \cdot \mu_{t}(p).
\] 
Hence, by Azuma–Hoeffding inequality, we have that 
\begin{align*}
\Pr\left[\Big\|\sum_{t\in \Gamma_{h_{\leq L}}} (p-X_{t}) \cdot \mathsf{1}[p_t = p] - \sum_{t\in \Gamma_{h_{\leq L}}} (p-X_{t}) \cdot q_{t}(p)\Big\|_1 \geq d\log(d)\sqrt{ T_{L}}\right] \leq \exp(-\log^2(d)/8).
\end{align*}
Therefore, we have 
\begin{align}
\E\left[\Big\|\sum_{t\in \Gamma_{h_{\leq L}}} (p-X_{t}) \cdot \mathsf{1}[p_t = p] - \sum_{t\in \Gamma_{h_{\leq L}}} (p-X_{t}) \cdot q_{t}(p)\Big\|_1\right] \leq &~ d\log(d)\sqrt{T_{L}} + \exp(-\log^2(d)/8)\cdot 2T_{L} \notag \\
\leq &~ 2d\log(d)\sqrt{T_{L}}.\label{eq:concentrate1}
\end{align}
Finally, we have
\begin{align*}
    &~ \E_{X}\left[\sum_{p\in P} \Big\|\sum_{t\in \Gamma_{h_{\leq L}}} (p-X_{t}) \cdot \mathsf{1}[p_t = p] - \sum_{t\in \Gamma_{h_{\leq L}}} (p-X_{t}) \cdot q_{t}(p)\Big\|_1\right] \\
    = &~\E_{X}\left[\sum_{p\in P_{h_{\leq L}}} \Big\|\sum_{t\in \Gamma_{h_{\leq L}}} (p-X_{t}) \cdot \mathsf{1}[p_t = p] - \sum_{t\in \Gamma_{h_{\leq L}}} (p-X_{t}) \cdot q_{t}(p)\Big\|_1\right] \\
    \leq &~ L\cdot 2d\log(d)\sqrt{T_{L}} \leq \eps T_{L}.
\end{align*}
Here the first step holds since for any $p \notin P_{h_{\leq L}}$, $q_t(p) = 0$ and $\mathsf{1}[p_t=p] = 0$ for $t \in \Gamma_{h_{\leq L}}$, the second step follows from Eq.~\eqref{eq:concentrate1}. The last step follows from the choice of $T_{L}$.

\end{proof}

\begin{proof}[Proof of Theorem \ref{thm:calibration}]
We first bound the expected calibration error. We have that
\begin{align*}
\mathsf{ECE}_{q, X} = &~ \E\left[\sum_{p \in \mK} \Big\| \sum_{t=1}^{T} (p_t - X_t) \cdot \mathsf{1}[p_t = p]\Big\|_1\right] \\
\leq &~  \E\left[\sum_{p \in \mK} \Big\| \sum_{t=1}^{T} (p_t - X_t) \cdot \mathsf{1}[p_t = p] - (p_t - X_t) \cdot q_t(p)\Big\|_1\right] + \DCE_{q, X}\\
\leq &~ \E\left[\sum_{p \in \mK} \sum_{h_{\leq L}\in [H]^{L}}\Big\|\sum_{t\in\Gamma_{h_{\leq L}}} (p_t - X_t) \cdot \mathsf{1}[p_t = p] - (p_t - X_t) \cdot q_t(p)\Big\|_1\right] + 4\eps T\\
\leq &~ H^{L} \cdot \eps T_{L} + 4\eps T = 5\eps T.
\end{align*}
The first step follows from the definition of expected calibration $\mathsf{ECE}_{q, X}$, the second step follows from triangle inequality and the definition of $\DCE_{q, X}$, the third step follows from triangle inequality and Lemma \ref{lem:distributional}, the fourth step follows from Lemma \ref{lem:sample-concentration}.

The toal number of days equals $T = (d^3/\eps^6)\cdot H^{L} = (d^3/\eps^6)\cdot (1/\eps^4)^{\log(d)/\eps^2} = d^{\wt{O}(1/\eps^2)}$

\end{proof}

\section{Polynomial lower bound for calibrated forecasting}


\calibrationLower*

\paragraph{Hard sequence} Algorithm \ref{algo:hard} depicts the hard sequence. Notably, the adversary is oblivious and it determines the outcome distribution at the beginning (Line \ref{line:hard1}--\ref{line:hard2} of Algorithm \ref{algo:hard}). Moreover, we assume the outcome distribution $p_t \in \Delta_d$ is known to the algorithm at the beginning of each day $t$ (the realized outcome $X_t \sim p_t$ is revealed to the algorithm after it makes the prediction).

\paragraph{Notations} Let $R$ be the number of levels.
For any $r \in [R]$, let $D_r = [(r-1) \cdot (d/R) + 1: r\cdot (d/R)]$ be $r$-th block of outcome. 
Let $K = d/R^2$, for any $k\in [K]$, let $D_{r, k} = [(r-1) \cdot (d/R) + (k-1)\cdot R + 1: (r-1) \cdot (d/R) + k\cdot R]$ be the $k$-th block in $D_r$. 
For any $j \in [R]$, let $\mathsf{1}_{r, k, j}$ be the one-hot vector whose $((r-1) \cdot(d/R) + k\cdot R + j)$-th coordinate equals one. 
For any $r \in [R-1]$ and $k_{\leq r} \in [K]^{r}$, let $I_{k_{\leq r}} = [\sum_{\tau \leq r}(k_\tau-1)K^{R-\tau-1} +1: \sum_{\tau}(k_\tau-1)K^{R-\tau-1} + K^{R-\tau-1}]$ be the $k_{\leq r}$-th time interval.

\begin{algorithm}[!htbp]
\caption{Hard sequence}
\label{algo:hard}
\begin{algorithmic}[1]
\State {\bf Parameters:} $R$, $K = d/R^2$ 
\For{$r = 1, \ldots, R-1$}\label{line:hard1}
\For{$k_{\leq r} \in [K]^{r}$}
\State Draw a random index $\tau_{k_{\leq r}} \in [R]$ \label{line:random_sample}
\EndFor\label{line:hard2}
\EndFor\\

\For{$t = (k_1, \ldots, k_{R-1})\in [K]^{R-1}$}\Comment{Day $t$}
\State The nature draws the outcome $X_t$ from $p_{t} = \frac{1}{R} (\sum_{r=1}^{R-1}\vec{1}_{r, k_{r}, \tau_{k_{\leq r}}} + \unif(D_{R}))$ 
\EndFor
\end{algorithmic}
\end{algorithm}

We would prove a lower bound for distributional calibration, which directly implies a lower bound for expected calibration. The hard sequence is drawn from a fixed distribution, so it suffices to consider a deterministic forecaster, whose distribution $\mu_t \in \Delta(\mathcal{K})$ is determined given the past outcome $X_1, \ldots, X_{t-1}$ and the outcome distribution $p_t$ at day $t$.

We first extend the definition of $\DCE$, make it well-defined with respect to any subset of predictions, outcome and time interval.
\begin{definition}
Given a deterministic distributional forecaster $\mu$ and a sequence of outcome $X$, for any time interval $I \subseteq [T]$, subset of predictions $P\subseteq \Delta_d$, and subset of outcome $D \subseteq [d]$,  define
\begin{align*}
\mathsf{DCE}_{\mu, X}(I, P, D) := \sum_{p\in P}\sum_{i \in D} \Big|\sum_{t\in I}(p(i)-X_t(i)) \cdot \mu_t(p)\Big|.
\end{align*}
That is, $\mathsf{DCE}_{\mu, X}(I, P, D)$ is the distributional calibration error within time interval $I$, over the set of predictions $P$ and outcome $D$, when the prediction strategy is $\mu$ and the outcome is $X$. 

Furthermore, we write $\mathsf{DCE}_{\mu}(I, P, D)$ to be the expected distributional calibration error when the outcome sequence $X$ is drawn from Algorithm \ref{algo:hard}, i.e.,
\begin{align*}
\mathsf{DCE}_{\mu}(I, P, D) := \E_{X}[\mathsf{DCE}_{\mu,X}(I, P, D)].
\end{align*}
\end{definition}
We note that $\DCE_{\mu, X}(I, P, D)$ (resp. $\DCE_{\mu}(I, P, D)$) is monotone with respect to the set of predictions $P$ and the set of outcome $D$, but it is not necessarily monotone with respect to the time interval $I$.

Let $\eps_r = (1/R)^{6(R - r + 1)}$ ($r \in [R]$) be the error parameter. Our main Lemma is stated as follow
\begin{lemma}
\label{lem:induction}
For any $r \in [R]$, $k_{<r}\in [K]^{r-1}$, $P_{k_{<r}} \subseteq \Delta_d$, we have that
\begin{align*}
\DCE_{\mu}(I_{k_{< r}}, P_{k_{< r}}, D_{\geq r}) \geq \eps_{r} \cdot \mu(P_{k_{< r}}, I_{k<r}).
\end{align*}
Here $\mu(P_{k_{< r}}, I_{k_{<r}}) = \sum_{p\in P_{k_{<r}}}\sum_{t\in  I_{k_{<r}}}\mu_t(p)$ is the total mass over predictions $P_{k_{<r}}$ in $I_{k_{< r}}$.
\end{lemma}
\begin{proof}
We prove by induction over $r = R, R-1, \ldots, 1$.

For the base case of $r = R$, at day $k_{<R}$, the outcome over $D_{R}$ is uniform. Hence, 
\begin{align*}
\DCE_{\mu}(I_{k_{< R}}, P_{k_{< R}}, D_{\geq R}) = &~ \E_{X_{k_{<R}}}\Big[\sum_{p\in P_{k_{< R}}}\sum_{i \in D_{R}} \Big|(p(i)-X_{k_{<R}}(i)) \cdot \mu_{k_{<R}}(p)\Big|\Big]\\
\geq &~ \sum_{p\in P_{k_{< R}}}\sum_{i \in D_{R}}\mu_{k_{<R}}(p) \cdot \Big(\frac{1}{d}|p(i)-1| + \frac{d-1}{d}|p(i)|\Big)\\
\geq &~ \sum_{p\in P_{k_{< R}}}\sum_{i \in D_{R}}\mu_{k_{<R}}(p) \cdot \frac{1}{d} \\
= &~ \frac{1}{R} \cdot \mu(P_{k_{< R}}, I_{k_{<R}}) \geq \eps_{R} \cdot \mu(P_{k_{< R}}, I_{k<R}).
\end{align*}
The second step holds since, at day $k_{<R}$, $X_{k_{<R}}(i) = 1$ happens with probability $1/d$ for $i \in D_R$ and $X_{k_{<R}}(i) = 0$ otherwise.  The third step holds since $\frac{1}{d}|p(i)-1| + \frac{d-1}{d}|p(i)| \geq \frac{1}{d}$ and the fourth step holds since $|D_{R}|/d = 1/R$.

For the induction step, suppose the claim holds up to $r+1$, then we prove it continues to hold for $r$.  We prove the claim holds for any fixed $k_{<r} \in [K]^{r-1}$. For simplicity of notation, we drop the subscript on $k_{<r}$ and we write $I := I_{k_{<r}}$, $P:= P_{k_{<r}}$, $\mu:= \mu(P, I)$, then our goal becomes
\begin{align}
\mathsf{DCE}_{\mu}(I, P, D_{\geq r}) \geq \eps_{r} \cdot \mu. \label{eq:goal}
\end{align}

Within the time interval $I$, there are $K$ blocks $I_1 = I_{k_{<r},1}, \ldots, I_{k} = I_{k_{<r},K}$, each contains $K^{R-r-1}$ days.
For any weight level $\alpha > 0$ and $k\in [K]$, define $P_{k}(\alpha) \subseteq P$ be the set of predictions that place at least $\alpha$ weight on outcome in $D_{r, >k}$, i.e., 
\begin{align*}
P_{k}(\alpha) := \left\{ p \in P: \sum_{i\in D_{r, > k}}p(i) \geq \alpha\right\}.
\end{align*}
Let $\beta(\alpha)$ be the total mass placed over $P_{k}(\alpha)$ during $I_{k}$ and sum over all $k\in [K]$, i.e.,
\begin{align*}
\beta(\alpha) = \sum_{k\in [K]}\mu(P_{k}(\alpha), I_{k})
\end{align*}
Intuitively, if $\beta(\alpha)$ is large, then it means the forecaster places large weight on outcome (in $D_r$) whose distribution has not been fixed, since each outcome (in $D_r$) has non-zero weight with probability only $1/R$ (within $I$), this tends to incur error. We formalize the observation as follow. 
\begin{lemma}
\label{lem:look-ahead}
For any $\alpha > 0$, if $\beta(\alpha) \geq (2\eps_{r}/\alpha) \cdot \mu$, then
\[
\mathsf{DCE}_{\mu}(I, P, D_{\geq r}) \geq \eps_{r} \cdot \mu.
\]
\end{lemma}
\begin{proof}
It is easy to see that $\emptyset = P_{K}(\alpha) \subseteq P_{K-1}(\alpha)\subseteq \cdots \subseteq P_1(\alpha)$.
For any $k\in [K]$, define
\begin{align*}
Q_k(\alpha) = P_{k}(\alpha) \setminus P_{k+1}(\alpha).
\end{align*}
It is easy to see that $\cup_{k\in [K]} Q_{k}(\alpha) = \cup_{k \in [K]} P_{k}(\alpha)$ and $\{Q_{k}(\alpha)\}_{k\in [K]}$ are disjoint. Moreover, we have 
\begin{align}
\sum_{k\in [K]}\sum_{p \in Q_k(\alpha)}\sum_{k'\leq k}\mu(p, I_{k'}) = &~ \sum_{k'\in [K]}\sum_{k \geq k'}\sum_{p\in P_k(\alpha)\setminus P_{k+1}(\alpha)}\mu(p, I_{k'}) \notag \\
= &~ \sum_{k'\in [K]}\sum_{p \in P_{k'}(\alpha)}\mu(p, I_{k'}) = \beta(\alpha) \geq (2\eps_r/\alpha)\mu \label{eq:look-ahead1}.
\end{align}
Here the first two step follows from the definition of $P_{k}(\alpha), Q_{k}(\alpha)$, the last step follows from the assumption.

Let $D_r'\subseteq D_{r}$ ($|D_r'| = K$) be the set outcome with non-zero weight within $I$, i.e., 
\[
D_r' := \{i \in D_{r}: p_t(i) > 0 \text{ for some } t \in I\}.
\]
For any $k \in [K]$, we have that 
\begin{align}
&~ \mathsf{DCE}_{\mu}(I, Q_k(\alpha), D_{r, > k})\notag \\
= &~ \E_{X}\Big[\sum_{p\in Q_k(\alpha)}\sum_{i\in D_{r, >k}}\Big|\sum_{t\in I}(p(i)-X_t(i)) \cdot \mu_t(p)\Big|\Big]\notag \\
\geq &~ \sum_{p\in Q_k(\alpha)}\sum_{i\in D_{r, >k}}\E_{X}\Big[\Big |\sum_{t\in I}(p(i)-X_t(i)) \cdot \mu_t(p) \Big|  \, | i \notin D_r'\Big]  \cdot \Pr[i \notin D_r']\notag \\
= &~ \sum_{p\in Q_k(\alpha)}\sum_{i\in D_{r, >k}}\E_{X}\Big[\Big |\sum_{t\in I}(p(i)-X_t(i)) \cdot \mu_t(p) \Big|  \, | i \notin D_r'\Big] \cdot (1 - 1/R)\notag \\
=&~ \sum_{p\in Q_k(\alpha)}\sum_{i\in D_{r, >k}}\E_{X}\Big[\Big |\sum_{t\in I}p(i) \cdot \mu_t(p) \Big|  \, | i \notin D_{r}'\Big] \cdot (1 - 1/R)\notag \\
\geq&~ \sum_{p\in Q_k(\alpha)}\sum_{i\in D_{r, >k}}\E_{X}\Big[\Big |\sum_{t\in I_{\leq k}}p(i) \cdot \mu_t(p) \Big|  \, | i \notin D_r'\Big] \cdot (1 - 1/R)\notag \\
=&~ \sum_{p\in Q_k(\alpha)}\sum_{i\in D_{r, >k}}\sum_{k'\leq k}\mu(p, I_{k'}) \cdot p(i) \cdot (1 - 1/R)\notag \\
\geq &~ \sum_{p\in Q_k(\alpha)}\sum_{k'\leq k}\mu(p, I_{k'}) \cdot \alpha \cdot (1-1/R) \label{eq:look-ahead2}
\end{align}
The third step holds since for any index $i \in D_{r}$, it appears in $D_r'$ with probability $1/R$ (Line \ref{line:random_sample} in Algorithm \ref{algo:hard}), the fourth step follows from $X_t(i) = 0$ when $i \notin D_r'$. 
The sixth step holds since for any $i \in D_{r, >k}$, the expected mass on $p$ in time interval $I_{\leq k}$ is independent of whether $i$ appears $D_r'$. The seventh step follows from $\sum_{i \in D_{r, >k}}p(i) \geq \alpha$ for $p \in Q_{k}(\alpha) \subseteq P_{k}(\alpha)$. 

Taking a summation over $k\in [K]$, we have that 
\begin{align*}
\DCE_{\mu}(I,P, D_{\geq r}) \geq &~ \sum_{k\in [K]}\DCE_{\mu}(I,Q_{k}(\alpha), D_{\geq r}) \geq \sum_{k\in [K]}\DCE_{\mu}(I,Q_{k}(\alpha), D_{r, > k})\\
\geq &~ \sum_{k\in [K]}\sum_{p\in Q_k(\alpha)}\sum_{k'\leq k}\mu(p, I_{k'}) \cdot \alpha \cdot (1-1/R) \\
\geq &~ (2\eps_r/\alpha) \cdot \mu \cdot \alpha \cdot(1-1/R) \geq \eps_{r}\mu.
\end{align*}
The first step follows from the definition of $\DCE_{\mu}(I,P, D_{\geq r})$ and $\{Q_k(\alpha)\}_{k\in [K]}$ are disjoint, the second step follows from the monotonicity over the outcome, the third step follows from Eq.~\eqref{eq:look-ahead2}, the fourth step follows from Eq.~\eqref{eq:look-ahead1}. This completes the proof.
\end{proof}

Now we are back to the proof of Lemma \ref{lem:induction}. In the rest of the proof, we prove by contradiction and assume Eq.~\eqref{eq:goal} does not hold. 

First, by Lemma \ref{lem:look-ahead}, if we take $\alpha = 1/R^2$, then we have 
\begin{align}
\beta(1/R^2) < (2\eps_r/(1/R^2))\cdot \mu =  (2R^2\eps_r) \cdot \mu.\label{eq:shift}
\end{align}
Define the probability mass $\rho$ as follows. For any prediction $p$ and day $t \in I_k$ (for some $k \in [K]$), $\rho_t(p)$ equals $\mu_t(p)$, unless at day $t$, $p \in P_{k}(1/R^2)$ (i.e., $p$ has more than $1/R^2$ weight on $D_{r, > k}$). Formally,
\begin{align}
\rho_t(p) = \left\{
\begin{matrix}
0 & p \in P_{k}(1/R^2), t \in I_k \text{ for some }k\in [K]\\
\mu_t(p) & \text{otherwise} \\
\end{matrix}
\right. \label{eq:rho}
\end{align}

Strictly speaking, $\rho_t$ is not a probability distribution (since it removes mass on $P_k(1/L^2)$), however one can still define $\DCE_{\rho}(P, I, D_{\geq r})$ in the same way. Since $\beta(1/R^2) < (2R^2\eps_r) \cdot \mu$, we know that $\rho$ is close to $\mu$ and we have 
\begin{align}
\DCE_{\mu}(I, P, D_{\geq r}) = &~ \E_{X}\Big[\sum_{p\in P}\sum_{i \in D_{\geq r}} \Big|\sum_{t\in I}(p(i)-X_t(i)) \cdot \mu_t(p)\Big|\Big]\notag \\
\geq &~ \E_{X}\Big[\sum_{p\in P}\sum_{i \in D_{\geq r}} \Big|\sum_{t\in I}(p(i)-X_t(i)) \cdot \rho_t(p)\Big|\Big]\notag \\
&~ - \E_{X}\Big[\sum_{p\in P}\sum_{i \in D_{\geq r}} \sum_{k\in [K]}\sum_{t\in I_{k}}\mu_t(p)\cdot |p(i) - X_t(i)|\cdot  \mathsf{1}[p \in P_{k}(1/R^2)]\Big].\notag \\
\geq &~ \E_{X}\Big[\sum_{p\in P}\sum_{i \in D_{\geq r}} \Big|\sum_{t\in I}(p(i)-X_t(i)) \cdot \rho_t(p)\Big|\Big] - 2\E_{X}\Big[\sum_{p\in P}\sum_{k\in [K]}\sum_{t\in I_{k}}\mu_t(p)\cdot  \mathsf{1}[p \in P_{k}(1/R^2)]\Big].\notag\\
=&~ \DCE_{\rho}(I, P, D_{\geq r}) - 2\beta(1/R^2)\notag \\
\geq &~ \DCE_{\rho}(I, P, D_{\geq r}) - 2R^2\eps_r \mu.\notag
\end{align}
Here the second step follows from the definition of $\rho_t$ (see Eq.~\eqref{eq:rho}) and the triangle inequality, the third step follows from $\sum_{i\in D_{\geq r}}|p(i) - X_{t}(i)| \leq 2$. The fourth step follows from the definition of $\beta$ and the last step follows from Eq.~\eqref{eq:shift}.
As we prove by contradiction, this implies
\begin{align}
\DCE_{\rho}(I, P, D_{\geq r}) \leq \DCE_{\mu}(I, P, D_{\geq r}) +2R^2\eps_r \mu <   (1+2R^2)\eps_r \mu \label{eq:diff}.
\end{align}

From now on, we would work on $\rho$ and we wish to bound $\DCE_{\rho}(I, P, D_{\geq r })$. We divide into a few steps.

{\noindent \bf Step 1.} Define $P_{\sml}:= \{p\in P: \sum_{i\in D_{r}}p(i) \leq 4/5R\}$, we prove that
\begin{align}
\rho(P_{\sml}, I) \leq 12R^3\eps_{r} \mu.
\label{eq:lower-step1}
\end{align}
Intuitively, predictions in $P_{\sml}$ assign too little weight on outcome $D_r$ so its total mass under $\rho$ can not be too much large. Formally, we have
\begin{align}
\DCE_{\rho}(I, P, D_{\geq r}) \geq &~ \DCE_{\rho}(I, P_{\sml}, D_{r}) \notag \\
= &~ \E_{X}\Big[\sum_{p\in P_{\sml}}\sum_{i \in D_{r}} \Big|\sum_{t\in I}(p(i)-X_t(i)) \cdot \mu_t(p)\Big|\Big] \notag\\
\geq &~ \E_{X}\Big[\sum_{p\in P_{\sml}}\sum_{i \in D_{r}} \sum_{t\in I}X_t(i) \cdot \mu_t(p)\Big] - \E_{X}\Big[\sum_{p\in P_{\sml}}\sum_{i \in D_{r}} \sum_{t\in I}p(i) \cdot \mu_t(p)\Big] \notag\\
\geq &~  \E_{X}\Big[\frac{1}{R}\sum_{p\in P_{\sml}} \sum_{t\in I} \mu_t(p)\Big] - \E_{X}\Big[\frac{4}{5R}\sum_{p\in P_{\sml}}\sum_{t\in I}p(i) \cdot \mu_t(p)\Big]\notag  \\
\geq &~ \frac{1}{5R} \rho(P_{\sml}, I). \label{eq:small}
\end{align}
The first step follows from the monotonicity of $\DCE$ on the prediction set and the outcome set, the second step follows from the definition of $\DCE_{\rho}(I, P_{\sml}, D_{r})$. The fourth step follows from $\E_{X_t}[\sum_{i\in D_r}p(i) \leq \frac{4}{5R}]$ for any $p\in P_{\sml}$ and $\E_{X_t}[\sum_{i\in D_r}X_{t}(i)] = 1/R$. 

Combining Eq.~\eqref{eq:small} and Eq.~\eqref{eq:diff}, we have proved Eq.~\eqref{eq:lower-step1}.

{\noindent \bf Step 2.} 
Define $P_{\smooth} = \{p\in P: \sum_{i\in D_{r, k}}p(i) \leq \frac{1}{10R}\forall k\in [K]\}$. That is, a prediction $p$ is in $P_{\smooth}$ if none of its block $\{D_{r, k}\}_{k\in [K]}$ has large weight.
We prove $\rho$ puts small mass on $P_{\smooth}$, i.e., 
\begin{align}
\rho(P_{\smooth},I) \leq 24 R^3 \eps_{r} \mu.\label{eq:lower-step2}
\end{align}

To this end, consider any prediction $p \in P_{\smooth}\setminus P_{\sml}$, there exists a block $\kappa(p) \in [K]$, such that 
\begin{align}
\frac{3}{5R} \geq \sum_{k > \kappa(p)}\sum_{i\in D_{r, k}}p(i) \geq \frac{1}{2R}. \label{eq:kappa}
\end{align}
By the the definition of $\rho$ (see Eq.~\eqref{eq:rho}), we have that, 
\begin{align}
\sum_{k \leq \kappa(p)}\rho(p, I_{k}) = 0 \quad\quad \forall p\in P_{\smooth}\setminus P_{\sml}.\label{eq:zero}
\end{align}
Now, we have that 
\begin{align*}
\DCE_{\rho}(I, P, D_{\geq r }) =&~ \E_{X}\Big[\sum_{p\in P}\sum_{i\in D_{\geq r}}\Big|\sum_{t\in I}(p(i)-X_t(i)) \cdot \rho_t(p)\Big|\Big]\\
\geq&~  \E_{X}\Big[\sum_{p\in P_{\smooth}\setminus P_{\sml}}\sum_{i\in D_{r, \leq \kappa(p)}}\Big|\sum_{t\in I}(p(i)-X_t(i)) \cdot \rho_t(p)\Big|\Big] \\
= &~ \E_{X}\Big[\sum_{p\in P_{\smooth}\setminus P_{\sml}}\sum_{i\in D_{r, \leq \kappa(p)}}\Big|\sum_{t\in I_{> \kappa(p)}}(p(i)-X_t(i)) \cdot \rho_t(p)\Big|\Big ]\\
= &~ \sum_{p\in P_{\smooth}\setminus P_{\sml}}\sum_{i\in D_{r, \leq \kappa(p)}}\sum_{t\in I_{> \kappa(p)}}p(i) \cdot \rho_t(p)\\
\geq &~ \sum_{p\in P_{\smooth}\setminus P_{\sml}}\sum_{t\in I_{> \kappa(p)}}  \rho_t(p) \cdot \frac{1}{5R} \\
= &~ \sum_{p\in P_{\smooth}\setminus P_{\sml}}\sum_{t\in I}  \rho_t(p) \cdot \frac{1}{5R} \\
= &~ \frac{1}{5R} (\rho(P_{\smooth}, I) - \rho(P_{\sml}, I)).
\end{align*}
The first step follows from the definition of $\DCE_{\rho}(I, P, D_{\geq r })$.
The third step follows from Eq.~\eqref{eq:zero}.
The fourth step holds since $X_{t}(i) = 0$ for $i \in  D_{r, \leq \kappa(p)}$ and $t \in I_{>\kappa(p)}$. 
The fifth step follows from $\sum_{i\in D_{r}}p(i) \geq \frac{4}{5R}$ for all $p \in P_{\smooth} \setminus P_{\mathsf{small}}$, and therefore, by Eq~\eqref{eq:kappa}, $\sum_{i\in D_{r, \leq\kappa(p)}}p(i) \geq \frac{4}{5R} - \frac{3}{5R} = \frac{1}{5R}$.
The sixth step follows from Eq.~\eqref{eq:zero}.

Combining with Eq.~\eqref{eq:lower-step1}, we have proved Eq.~\eqref{eq:lower-step2}.

{\noindent \bf Step 3.} Now consider the set $P' = P\setminus (P_{\mathsf{small}} \cup P_{\smooth})$, for any prediction $p \in P'$, there must exist a block $k \in [K]$, such that $\sum_{i\in D_{r, k}}p(i) \geq \frac{1}{10R}$ (since $p \notin P_{\smooth}$). Define $\eta(p) \in [K]$ be the largest such block. First, by the definition of Eq.~\eqref{eq:rho}, we have that
\begin{align}
\sum_{k < \eta(p)}\rho(p, I_{k}) = 0 \quad \quad  \forall p\in P'.\label{eq:zero2}
\end{align}
Hence, combining Eq.~\eqref{eq:shift}\eqref{eq:lower-step1}\eqref{eq:lower-step2}\eqref{eq:zero2}, we have that 
\begin{align}
\sum_{p \in P'}\sum_{k \geq \eta(p)}\rho(p, I_{k}) = &~ \sum_{p \in P'}\sum_{k \in [K]}\rho_{k}(p) =  \rho - \rho(P_{\smooth}, I) - \rho(P_{\sml}, I)\notag\\
\geq &~ \mu - 2R^2\eps_r \mu -12R^3\eps_r \mu - 24R^3\eps_r \mu \geq \mu - 40R^3\eps_r\mu. \label{eq:lower-step3}
\end{align}
That, $p\in P'$ takes the most mass from $\rho$ and they all appear on or after $\eta(p)$.
We further divide into two sub-steps.

{\noindent \bf Step 3.1} First, we prove
\begin{align}
\sum_{p \in P'}\rho(p, I_{\eta(p)}) \geq \frac{1}{20}\mu \label{eq:lower-step3-1}.
\end{align}
To see this, we have
\begin{align*}
\DCE_{\rho}(I, P, D_{\geq r}) \geq &~\E_{X}\Big[\sum_{p\in P'}\sum_{i\in D_{r, \eta(p)}}\Big|\sum_{t\in I}(p(i)-X_t(i)) \cdot \rho_t(p)\Big|\Big]\\
= &~ \E_{X}\Big[\sum_{p\in P'}\sum_{i\in D_{r, \eta(p)}}\Big|\sum_{k \geq \eta(p)}\sum_{t\in I_k}(p(i)-X_t(i)) \cdot \rho_t(p)\Big|\Big]\\
\geq &~ \E_{X}\Big[\sum_{p\in P'}\sum_{i\in D_{r, \eta(p)}}\sum_{k \geq \eta(p)}\sum_{t\in I_k}p(i) \cdot \rho_t(p) - \sum_{p\in P'}\sum_{i\in D_{r, \eta(p)}}\sum_{k \geq \eta(p)}\sum_{t\in I_k}X_t(i)\cdot \rho_t(p)\Big]\\
\geq &~ \frac{1}{10R} \sum_{p \in P'}\sum_{k \geq \eta(p)}\rho(p, I_{k}) - \E_{X}\Big[\sum_{p\in P'}\sum_{i\in D_{r, \eta(p)}}\sum_{t\in I_{\eta(p)}}X_{t}(i)\cdot \rho_t(p)\Big].\\
= &~ \frac{1}{10R} \sum_{p \in P'}\sum_{k \geq \eta(p)}\rho(p, I_{k}) - \frac{1}{R}\sum_{p\in P'}\rho(p, I_{\eta(p)})\\
\geq &~ \frac{1}{10R}\mu - 4R^2\eps_r\mu - \frac{1}{R}\sum_{p\in P'}\rho(p, I_{\eta(p)})
\end{align*}
The first step follows from the definition of $\DCE_{\rho}(I, P, D_{\geq r})$, the second step follows from Eq.~\eqref{eq:zero2}.
The fourth step follows from $\sum_{i\in D_{r, \eta(p)}}p(i)\geq \frac{1}{10R}$ for any $p\in P'$, and $X_{t}(i) = 0$ for any $i \in D_{r, \eta(p)}$ and $t \in I_{>\eta(p)}$.
The fifth step follows from $\E_{X}[\sum_{i\in D_{r, \eta(p)}}X_{t}(i)] = 1/R$ for any $t \in I_{\eta(p)}$.
The last step follows from Eq.~\eqref{eq:lower-step3}. 

Combining with Eq.~\eqref{eq:diff}, we have proved Eq.~\eqref{eq:lower-step3-1}.

{\noindent\bf Step 3.2} Define $P'' := \{p\in P': \sum_{k > \eta(p)}\sum_{i\in D_{r, k}}p(i) < R^3\eps_{r}\}$.
By Lemma \ref{lem:look-ahead} and taking $\alpha = R^3\eps_{r}$, we have that 
\begin{align}
\sum_{p\in P'\setminus P''}I(p, I_{\eta(p)}) \leq (2\eps_r/R^3\eps_r)\cdot\mu =  \frac{2}{R^3} \cdot \mu \label{eq:truncate}
\end{align}
Combining Eq.~\eqref{eq:lower-step3-1}\eqref{eq:truncate}, this implies that
\begin{align}
\sum_{p \in P''}\rho(p, I_{\eta(p)}) \geq \frac{1}{20}\mu - \frac{2}{R^3} .\mu\label{eq:lower_ind}
\end{align}

Define $P_{k}'':= \{p \in P'', \eta(p) = k\}$, we note that $\cup_{k\in [K]} P_{k}'' = P''$ and $\{P_{k}''\}_{k\in [K]}$ are disjoint.
Now we can apply the inductive hypothesis
\begin{align}
\sum_{k\in [K]}\DCE_{\rho}(I_{k}, P_{k}'', D_{> r}) \geq &~ \sum_{k\in [K]}\eps_{r+1}\cdot \rho(P_{k}'', I_{k})\notag\\
= &~ \eps_{r+1} \sum_{p\in P''}\rho(p, I_{\eta(p)}) \geq \eps_{r+1} \Big(\frac{1}{20}\mu - \frac{2}{R^3} \mu\Big).\label{eq:case1}
\end{align}
The first step follows from the induction, the third step follows from Eq.~\eqref{eq:lower_ind}.

We divide into two cases. 

{\bf Case 1.} Suppose $\sum_{k\in [K]}\rho(P_{k}'', I_{>k}) \leq \frac{1}{80}\eps_{r+1}\mu$. In this case, we can bound the calibration error over outcome in $D_{>r}$. That is,
\begin{align*}
\DCE_{\rho}(I, P, D_{\geq r}) \geq &~ \DCE_{\rho}(I, P'', D_{> r})\\
 = &~ \E_{X}\Big[\sum_{k\in [K]}\sum_{p\in P_{k}''}\sum_{i \in D_{>r}}\Big|\sum_{t\in I} (p(i) - X_{t}(i))\cdot \rho_t(p)\Big|\Big]\\
= &~ \E_{X}\Big[\sum_{k\in [K]}\sum_{p\in P_{k}''}\sum_{i \in D_{>r}}\Big|\sum_{t\in I_{\geq k}} (p(i) - X_{t}(i))\cdot \rho_t(p)\Big|\Big]\\
\geq &~ \E_{X}\Big[\sum_{k\in [K]}\sum_{p\in P_{k}''}\sum_{i \in D_{>r}}\Big|\sum_{t\in I_{k}} (p(i) - X_{t}(i))\cdot \rho_t(p)\Big| - \Big|\sum_{t\in I_{> k}} (p(i) - X_{t}(i))\cdot \rho_t(p)\Big|\Big]\\
\geq &~ \sum_{k\in [K]}\DCE_{\rho}(I_{k}, P_{k}'', D_{> r})- 2\sum_{k\in [K]}\rho(P_{k}'', I_{>k}) \\
\geq &~ \eps_{r+1} \Big(\frac{1}{20}\mu - \frac{2}{R^3} \mu\Big) - \frac{1}{40}\eps_{r+1}\mu \geq \frac{1}{50}\eps_{r+1}\mu.
\end{align*}
The first two steps follow from the definition of $\DCE_{\rho}(I, P, D_{\geq r})$, $\{P_{k}''\}_{k\in [K]}$ are disjoint and $\cup_{k\in [K]}P_{k}'' = P''$.
The third step holds since $\rho_t(p) = 0$ for $p\in P_{k}''$ and $t \in I_{<k}$ (see Eq.~\eqref{eq:zero2}).
The fourth step follows from the triangle inequality.
The fifth step follows from the definition of $\DCE_{\rho}(I_{k}, P_{k}'', D_{> r})$.
The sixth step follows from Eq~\eqref{eq:case1} and the assumption of Case 1. This contradicts with Eq.~\eqref{eq:diff}.

{\bf Case 2.} Suppose $\sum_{k\in [K]}\rho(P_{k}'', I_{>k})  > \frac{1}{80}\eps_{r+1}\mu$. Then we bound the calibration error over outcome in $D_{r}$:
\begin{align*}
\DCE_{\rho}(I, P, D_{\geq r}) \geq &~  \DCE_{\rho}(I, P'', D_{r}) \\
\geq &~ \E_{X}\Big[\sum_{k\in [K]}\sum_{p\in P_{k}''}\sum_{i \in D_{r, > k}}\Big|\sum_{t\in I} (p(i) - X_{t}(i))\cdot \rho_t(p)\Big|\Big]\\
= &~ \E_{X}\Big[\sum_{k\in [K]}\sum_{p\in P_{k}''}\sum_{i \in D_{r, >k}}\Big|\sum_{t\in I_{\geq k}} (p(i) - X_{t}(i))\cdot \rho_t(p)\Big|\Big]\\
\geq &~ \E_{X}\Big[\sum_{k\in [K]}\sum_{p\in P_{k}''}\sum_{i \in D_{r, >k}}\sum_{t\in I_{\geq k}} X_{t}(i)\rho_t(p) - p(i) \rho_t(p)\Big]\\
= &~ \sum_{k\in [K]}\rho(P_{k}'', I_{> k}) \cdot \frac{1}{R} -  \sum_{k\in [K]}\rho(P_{k}'', I_{\geq k})\cdot R^{3}\eps_r \\
\geq &~  \frac{1}{80}\eps_{r+1}\mu \cdot \frac{1}{R} - \mu\cdot R^3\cdot \eps_r \geq \frac{1}{100R}\eps_{r+1} \mu. 
\end{align*}
The first two steps follow from the definition of $\DCE_{\rho}(I, P, D_{\geq r}, \rho)$, $\{P_{k}''\}_{k\in [K]}$ are disjoint and $\cup_{k\in [K]}P_{k}'' = P''$.
The third step holds since $\rho_t(p) = 0$ for $p\in P_{k}''$ and $t \in I_{<k}$ (see Eq.~\eqref{eq:zero2}).
The fifth step holds since $\E_{X}[\sum_{i \in D_{r, >k}}X_t(i)] = 1/R$ for any $t\in I_{>k}$ and $\sum_{i\in D_{r, >k}}p(i) < R^{3}\eps_r$ for any $p \in P_{k}''$ (see the definition of $P''$ and $P_k''$).
The seventh step follows from the assumption of Case 2. This contradicts with Eq.~\eqref{eq:diff}.
\end{proof}

\begin{proof}[Proof of Theorem \ref{thm:lower}]
By Lemma \ref{lem:induction}, taking $r = 1$ and $P = \Delta_{d}$, we have that the distributional calibration error of any algorithm obeys $\DCE_{\mu} \geq \eps_{1} T$. Note that $T = K^{R-1} = (d/R^2)^{R-1}$ and $\eps_1 = R^{-O(R)}$, this suggests we can take $R = \frac{\log(1/\eps)}{\log\log(1/\eps)}$ and prove that $\eps$-calibration can only be obtained after $(d/R^2)^{R-1} = d^{\wt{\Omega}(\log(1/\eps))}$ days.
\end{proof}

\section*{Aknowledgement}
The author would like to thank Aviad Rubinstein and Xi Chen for useful discussion.

\bibliographystyle{alpha}
\bibliography{ref}

\end{document}